\documentclass[12pt,a4paper]{article}

\usepackage{graphicx} % more modern
\usepackage{subfigure}

\usepackage{algorithm}
\usepackage{algorithmic}
\usepackage{hyperref}
\usepackage{amsmath,amstext,amsfonts,amsxtra,amssymb,latexsym,bm,color}
\usepackage{psfrag,balance}
\usepackage{array}

\newcommand{\x}{\mathbf{x}}

\newcommand{\w}{\mathbf{w}}

\renewcommand{\u}{\mathbf{u}}
\newcommand{\U}{\mathbf{U}}
\newcommand{\z}{\mathbf{z}}

\newcommand{\bs}[1]{{\mathbf{#1}}}

\begin{document}

\title{ConeRANK: Ranking as Learning Generalized Inequalities}

\author{Truyen T. Tran$\dagger$ and Duc-Son Pham$\ddagger$\\
	$\dagger$ Center for Pattern Recognition and Data Analytics (PRaDA), \\
		Deakin University, Geelong, VIC, Australia \\	
	$\ddagger$ Department of Computing, Curtin University, Western Australia \\
	Email: \texttt{truyen@vietlabs.com} \texttt{dspham@ieee.org}}

\maketitle

\newtheorem{assumption}{\bf Assumption}
\newtheorem{lemma}{\bf Lemma}
\newtheorem{definition}{Definition}
\newtheorem{theorem}{\bf Theorem}
\newtheorem{proposition}{\bf Proposition}
\newtheorem{corollary}{\bf Corollary}
\newtheorem{observation}{\bf Observation}
\newenvironment{proof}{{\it Proof.}}

\begin{abstract}
We propose a new data mining approach in ranking documents based on the concept of cone-based generalized inequalities between vectors. A partial ordering between two vectors is made with respect to a proper cone and thus learning the preferences is formulated as learning proper cones. A pairwise learning-to-rank algorithm (ConeRank) is proposed to learn a non-negative subspace, formulated as a polyhedral cone, over document-pair differences. The algorithm is regularized by controlling the `volume' of the cone. The experimental studies on the latest and largest ranking dataset LETOR 4.0 shows that ConeRank is competitive against other recent ranking approaches.
\end{abstract}

\section{Introduction}

\emph{Learning to rank} in information retrieval (IR) is an emerging subject \cite{cohen1999learning,joachims2002optimizing,freund2004eba,burges2005learning,cao2007learning} with great promise to improve the retrieval results by applying machine learning techniques to learn the document relevance with respect to a query.  Typically, the user submits a query and the system returns a list of related documents. We would like to learn a ranking function that outputs the position of each returned document in the decreasing order of relevance. 

Generally, the problem can be studied in the supervised learning setting, in that for each query-document pair, there is an extracted feature vector and a position label in the ranking. The feature can be either \emph{query-specific} (e.g. the number of matched keywords in the document title) or \emph{query-independent} (e.g. the PageRank score of the document, number of in-links and out-links, document length, or the URL domain). In training data, we have a  groundtruth ranking per query, which can be in the form of a relevance score assigned to each document, or an ordered list in decreasing level of relevance.

The learning-to-rank problem has been approached from different angles, either treating the ranking problem as ordinal regression \cite{herbrich2lmr,chu2006gpo}, in which an ordinal label is assigned to a document,  as pairwise preference classification \cite{joachims2002optimizing,freund2004eba,burges2005learning} or as a listwise permutation problem \cite{plackett1975analysis,cao2007learning}.

We focus on the pairwise approach, in that ordered pairs of document per query will be treated as training instances, and in testing, predicted pairwise orders within a query will be combined to make a final ranking.  The advantage of this approach is that many existing powerful binary classifiers that can be adapted with minimal changes - SVM \cite{joachims2002optimizing},
boosting \cite{freund2004eba}, or logistic regression \cite{burges2005learning} are some choices. 

We introduce an entirely new perspective based on the concept of cone-based \emph{generalized inequality}. More specifically, the inequality between two multidimensional vectors is defined with respect to a cone. Recall that a cone is a geometrical object in that if two vectors belong to the cone, then any non-negative linear combination of the two vectors also belongs to the cone. Translated into the framework of our problem, this means that given a cone $\mathcal{K}$, when document $l$ is ranked higher than document $m$, the feature vector $\mathbf{x}_l$ is `greater' than the feature vector $\mathbf{x}_m$ with respect to  $\mathcal{K}$ if $\mathbf{x}_l-\mathbf{x}_m \in \mathcal{K}$. Thus, given a cone, we can find  the correct order of preference for any given document pair. However,  since the cone $\mathcal{K}$ is not known in advance, it needs to be estimated from the data.  Thus, in our paper, we consider polyhedral cones constructed from basis vectors and  propose a method for learning the cones via the estimation of this set of basis vectors.

This paper makes the following contributions:
\begin{itemize}
\item A novel formulation of the learning to rank problem, termed as ConeRank, from the angle of cone learning and generalized inequalities;
\item A study on the generalization bounds of the proposed method;
\item Efficient online cone learning algorithms,  scalable with large datasets; and,
\item An evaluation of the algorithms on the latest LETOR 4.0 benchmark dataset~\footnote{Available at: 
	http://research.microsoft.com/en-us/um/beijing/projects/letor/letor4dataset.aspx}.
\end{itemize}

\begin{figure}
\includegraphics[width=0.95\linewidth]{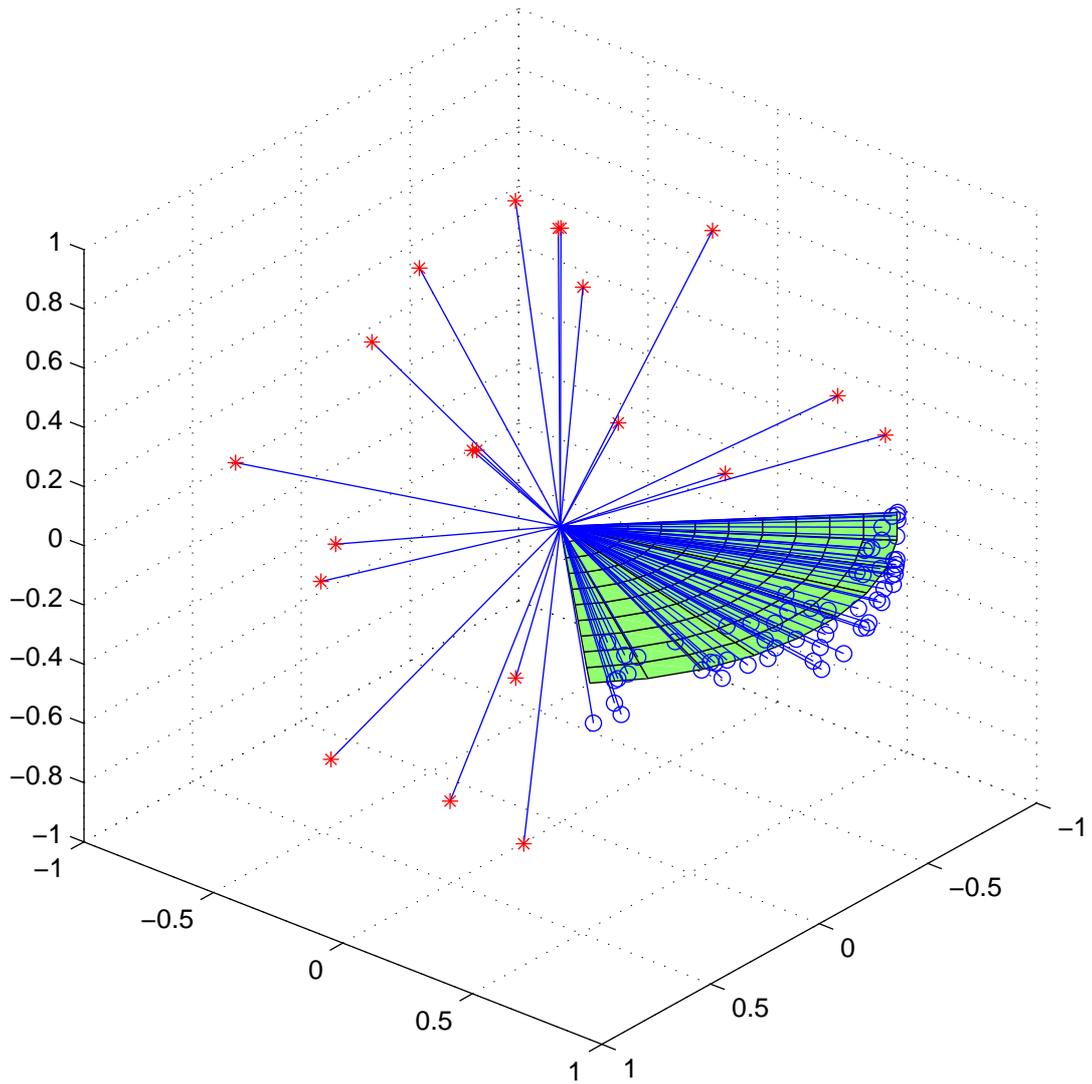}
\vspace{-1.5cm}
\caption{\label{FIG_3DCONE} Illustration of ConeRank. Here the pairwise differences are distributed in 3-dimensional space, most of which however lie only on a surface and can be captured most effectively by a `minimum' cone plotted in green. Red stars denotes noisy samples.}
\end{figure}

\section{Previous Work}

Learning-to-rank is an active topic in machine learning, although ranking and permutations have been studied widely in statistics. One of the earliest paper in machine learning is perhaps \cite{cohen1999learning}. The seminal paper \cite{joachims2002optimizing} stimulates much subsequent research. Machine learning methods extended to ranking can be divided into:

\emph{Pointwise approaches}, that include methods  such as ordinal regression \cite{herbrich2lmr,chu2006gpo}. Each query-document pair is assigned a ordinal label, e.g. from the set $\{0,1,2,...,L\}$. This simplifies the problem as we do not need to worry about the exponential number of permutations. The complexity is therfore linear in the number of query-document pairs. The drawback is that the ordering relation between documents is not explicitly modelled.

{\it Pairwise approaches}, that span preference to binary classification \cite{joachims2002optimizing,freund2004eba,burges2005learning} methods, where the goal is to learn a classifier that can separate two documents (per query). This casts the ranking problem into a standard classification framework, wherein many algorithms are readily available. 
%Here, we follow the pairwise approach, but from the perspective of conic programming.
The complexity is quadratic in number of documents per query and linear in number of queries.

\emph{Listwise approaches}, modelling the distribution of permutations \cite{cao2007learning}. The ultimate goal is to model a full distribution of all permutations, and the prediction phase outputs the most probable permutation. In the statistics community, this problem has been long addressed \cite{plackett1975analysis}, from a different angle. The main difficulty is that the number of permutations is exponential and thus approximate inference is often used.

However, in IR, often the evaluation criteria is different from those employed in learning. So there is a trend to optimize the (approximate or bound) IR metrics \cite{Cossock_Zhang08}.

\section{Proposed Method}

\subsection{Problem Settings}
We consider a training set of $P$ queries $q_1,q_2,\ldots,q_P$ randomly sampled from a query space $\mathcal{Q}$ according to some distribution ${P}_{\mathcal{Q}}$. Associated with each query $q$ is a set of documents represented as pre-processed feature vectors $\{\x^q_1,\x^q_2\ldots\}, \x^q_l\in\mathbb{R}^N$ with relevance scores $r^q_1,r^q_2,\ldots$ from which ranking over documents can be based. We note that the values of the feature vectors may be query-specific and thus the same document can have different feature vectors according to different queries. Document $\x^q_l$ is said to be more preferred than document $\x^q_m$ for a given query $q$ if $r^q_l>r^q_m$ and vice versa. In the {\it pairwise} approach,  pursued in this paper, equivalently we learn a ranking function $f$ that takes input as a pair of different documents $\x^q_l,\x^q_m$ for a given query $q$ and returns a value $y\in\{+1,-1\}$ where $+1$ corresponds to the case where $\x^q_l$ is ranked above $x^q_m$ and vice versa. For notational simplicity, we may drop the superscript $^q$ where there is no confusion.

\subsection{Ranking as Learning Generalized Inequalities}

In this work, we consider the ranking problem from the viewpoint of generalized inequalities. In convex optimization theory \cite[p.34]{Boyd_Vandenberghe04}, a generalized inequality $\succ_{\mathcal{K}}$ denotes a partial ordering induced by a proper cone $\mathcal{K}$, which is convex, closed, solid, and pointed:
\[\x_{l}\succ_{\mathcal{K}}\x_{m}\Longleftrightarrow\x_{l}-\x_{m}\in\mathcal{K}.\]
Generalized inequalities satisfy many properties such as preservation under addition, transitivity, preservation under non-negative scaling, reflexivity, anti-symmetry, and preservation under limit. 

We propose to learn a generalized inequality or, equivalently, a proper cone $\mathcal{K}$ that best describes the training data (see Fig. \ref{FIG_3DCONE} for an illustration). 
%An important property of generalized inequality is {\it scale-invariance}. 
Our important assumption is that this proper cone, which induces the generalized inequality, is not query-specific and thus prediction can be used for unseen queries and document pairs coming from the same distributions. 

From a fundamental property of convex cones, if $\z\in\mathcal{K}$ then $w\z\in\mathcal{K}$ for all $w>0$, and any non-negative combination of the cone elements also belongs to the cone, i.e. if $\u_{k}\in\mathcal{K}$ then $\sum_{k}w_{k}\u_{k}\in\mathcal{K}, \forall w_{k}>0$.

In this work, we restrict our attention to {\it polyhedral} cones for the learning of generalized inequalities. A polyhedral cone is a polyhedron and a cone. A polyhedral cone can be defined as sum of rays or intersection of halfspaces. We construct the polyhedral cone $\mathcal{K}$ from `basis' vectors $\U=[\u_1,\u_2,\ldots,\u_K]$. They are the extreme vectors lying on the intersection of hyperplanes that define the halfspaces. Thus, the cone $\mathcal{K}$ is a conic hull of the basis vectors and is completely specified if the basis vectors are known. A polyhedral cone with $K$ basis vectors is said to have an order $K$ if one basis vector cannot be expressed as a conic combination of the others. It can be verified that under these regular conditions, a polyhedral cone is a proper cone and thus can induce a generalized inequality. We thus propose to learn the basis vectors $\u_k,k=1,\ldots,K$ for the characterization of $\mathcal{K}$. % We however restrict $K\leq N$ for simplicity.

A projection of $\z$ onto the cone $\mathcal{K}$, denoted by $\mathsf{P}_{\mathcal{K}}(\z)$, is generally defined as some $\z'\in\mathcal{K}$ such that a certain criterion on the distance between $\z$ and $\z'$ is met. As $\z'\in\mathcal{K}$, it follows that it admits a conic representation $\z'=\sum_{k=1}^{K}w_k\u_k = \U\w, \ w_k\geq 0$. By restricting the order $K\leq N$, it can be shown that when $\U$ is full-rank then the conic representation is unique.

Define an ordered document-pair ($l,m$) difference as $\z=\x_l-\x_m$ where, without loss of generality, we assume that $r_l\geq r_m$. The linear representation of $\z'\in\mathcal{K}$ can be found from  
\begin{eqnarray}
	\label{EQU_PROJECTION}
 	\min_{\w} &&  \|\z - \U\w\|^2_2,  \ \ \w \geq \bs{0} 
% 	\mbox{s.t} && \left\{ 
% 			\begin{array}{c}
% 				\w \geq 0 \\
% 				\| \w \|_0 \leq \rho				
% 			\end{array}	
% 		\right. \nonumber
\end{eqnarray}
where the inequality constraint is element-wise. It can be seen that $\mathsf{P}_{\mathcal{K}}(\z)=\z, \forall \z\in\mathcal{K}$. Otherwise, if $\z\not\in\mathcal{K}$ then it can be easily proved by contradiction that the solution $\w$ is such that $\U\w$ lies on a facet of $\mathcal{K}$. Let $\mathcal{K}^{-}$ be the cone with the basis $-\U$ then it can be easily shown that if $\z\in\mathcal{K}^{-}$ then $\mathsf{P}_{\mathcal{K}}(\z) = \bs{0}$. 

Returning to the ranking problem, we need to find a $K$-degree polyhedral cone $\mathcal{K}$ that captures most of the training data. Define the $\ell_2$ distance from $\z$ to $\mathcal{K}$ as $d_{\mathcal{K}}(\z) = \|\z-\mathsf{P}_{\mathcal{K}}(\z)\|_2$ then we define the document-pair-level loss as
	\begin{eqnarray}
		\label{EQU_LOSS}
	 	l(\mathcal{K};\z,y) = d_{\mathcal{K}}(\z)^2.
	\end{eqnarray}
Suppose that for a query $q$, a set of document pair differences $S_q = \{\z^q_1,\ldots,\z^q_{n_q}\}$ with relevance differences $\phi^q_1,\ldots,\phi^q_{n_q}, \phi^q_j>0$ can be obtained. Following \cite{Lan_etal08}, we define the empirical query-level loss as
	\begin{eqnarray}
	 	\hat{L}(\mathcal{K};q,S_q) = \frac{1}{n_q} \sum_{j=1}^{n_q}  l(\mathcal{K};\z^q,y^q). 
	\end{eqnarray}
For a full training set of $P$ queries and $S=\{S_{q_1},\ldots,S_{q_P}\}$ samples, we define the query-level empirical risk as
	\begin{eqnarray}
		\label{EQU_RISK}
		\hat{R}(\mathcal{K};S) = \frac{1}{P}\sum_{i=1}^{P} \hat{L}(\mathcal{K};q_i,S_{q_i}).
	\end{eqnarray}
Thus, the polyhedral cone $\mathcal{K}$ can be found from minimizing this query-level empirical risk. Note that even though other performance measures such as  mean average precision (MAP) or normalized discounted cumulative gain (NDCG) is the ultimate assessment, it is observed that good empirical risk often leads to good MAP/NDCG and simplifies the learning. We next discuss some additional constraints for the algorithm to achieve good generalization ability.

% 
% \[
% \min\sum_{lmk}\phi[w_{lmk}\ne0]\]
% \begin{eqnarray*}
% \mbox{subject to:}\\
% w_{lmk} & \ge & 0\,\,\forall l,m,k\\
% \y_{lm} & = & \sum_{k=1}^{K}w_{lmk}\y_{k}\,\,\,\\
% \forall\,\y_{lm} & = & \x_{l}-\x_{m}\\
% \x_{l} & \succ & \x_{m}\end{eqnarray*}
% 
% 
% This difficult optimisation problem can be relaxed as follows\begin{equation}
% \min_{\{\y_{k}\}}\min_{\{\w_{lm}\}|w_{lmk}>0}\frac{1}{P}\sum_{lm|\x_{l}\succ\x_{m}}L(\x_{l},\x_{m};\w_{lm},\y)\label{eq:opt-prob}\end{equation}
% where\begin{equation}
% L(\x_{l},\x_{m};\w_{lm},\y)=\left(\left\Vert \x_{l}-\x_{m}-\sum_{k=1}^{K}w_{lmk}\y_{k}\right\Vert _{2}^{2}+\gamma\sum_{k}w_{lmk}\right)\label{eq:loss}\end{equation}
% and $P$ is the number of vector pairs and $\gamma>0$.
% 
% Note that, this is essentially matrix factorisation where the elements
% of the coefficient matrix are non-negative.

\subsection{Modification}

{\it Normalization.} Using the proposed approach, the direction of the vector $\z$ is more important than its magnitude. However, at the same time, if the magnitude of $\z$ is small it is desirable to suppress its contribution to the objective function.  We thus propose the normalization of input document-pair differences as follows
\begin{equation}
	\z\leftarrow \rho\z /(\alpha +\|\z\|_2), \ \ \alpha,\rho > 0.
\end{equation}
The constant $\rho$ is simply the scaling factor whilst $\alpha$ is to suppress the noise when $\|\z\|_2$ is too small. With this normalization, we note that 
\begin{eqnarray}
 	\label{EQU_NORM}
	\|\z\|_2 \leq \rho.
\end{eqnarray}

{\it Relevance weighting.} In the current setting, we consider all ordered document-pairs equally important. This is however a disadvantage because the cost of the mismatch between the two vectors which are close in rank is less than the cost between those distant in rank. To address this issue, we propose an extension of (\ref{EQU_LOSS}) 
\begin{equation}
	\label{EQU_LOSS2}
	l(\mathcal{K};\z,y) = \phi d_{\mathcal{K}}(\z)^2.
\end{equation}
where $\phi>0$ is the corresponding ordered relevance difference.

{\it Conic regularization.} From statistical learning theory \cite[ch.4]{Scholkopf_Smola02}, it is known that in order to obtain good generalization bounds, it is important to restrict the hypothesis space from which the learned function is to be found. Otherwise, the direct solution from an unconstrained empirical risk minimization problem is likely to overfit and introduces large variance (uncertainty). In many cases, this translates to controlling the complexity of the learning function. In the case of support vector machines (SVMs), this has the intuitive interpretation of maximizing the margin, which is the inverse of the norm of the learning function in the Hilbert space.

In our problem, we seek a cone which captures most of the training examples, i.e. the cone that encloses the conic hull of most training samples. In the SVM case, there are many possible hyperplanes that separates the samples without a controlled margin. Similarly, there is also a large number of polyhedral cones that can capture the training samples without further constraints. In fact, minimizing the empirical risk will tend to select the cone with larger solid angle so that the training examples will have small loss (see Fig. \ref{FIG_REG}). In our case, the complexity is translated roughly to the size (volume) of the cone. The bigger cone will likely overfit (enclose) the noisy training samples and thus reduces generalization. Thus, we propose the following constraint to {\it indirectly} regularize the size of the cone
	\begin{eqnarray}
			\label{EQU_WCON}
			0 \leq \lambda_l \leq \| \w \|_1 \leq \lambda_u, \ \ \w \geq \bs{0}
	\end{eqnarray}
where $\w$ is the coefficients defined as in (\ref{EQU_PROJECTION}) and for simplicity we set $\lambda_l=1$. To see how this effectively controls $\mathcal{K}$, consider a 2D toy example in Fig. \ref{FIG_REG}. If $\lambda_u=1$, the solution is the cone $\mathcal{K}_1$. In this case, the loss of the positive training examples (within the cone) is the distance from them to the simplex define over the basis vectors $\u_1,\u_2$ (i.e. $\{\z: \z = \lambda\u_1 + (1-\lambda)\u_2, 0\leq \lambda\leq 1\}$) and the loss of the negative training example is the distance to the cone. With the same training examples, if we let $\lambda_u>1$ then there exists a cone solution $\mathcal{K}_2$ such that all the losses are effectively zero. In particular, for each training example, there exists a corresponding $\|\w\|_1=\lambda$ such that the corresponding simplex $\{\z: \z= w_1\u_1+w_2\u_2, w_1+w_2=\lambda\}$, passes all positive training examples.

Finally, we note that as the product $\U\w^{q_i}_j$ appears in the objective function and that both $\U$ and $\w^{q_i}_j$ are variables then there is a scaling ambiguity in the formulation. We suggest to address this scale ambiguity by considering the norm constraint $\|\u_k\|_2=c > 0$ on the basis vectors.

\begin{figure}[!h]
		\psfrag{TAGk1}[c][][1][0]{$\mathcal{K}_1$}
		\psfrag{TAGk2}[c][][1][0]{$\mathcal{K}_2$}
		\psfrag{TAGu1}[c][][1][0]{$\bs{u}_1$}
		\psfrag{TAGu2}[c][][1][0]{$\bs{u}_2$}
	   	\psfrag{TAGs1}[c][][1][0]{$\bs{z}_1$}
		\psfrag{TAGs2}[c][][1][0]{$\bs{z}_2$}
		\psfrag{TAGs3}[c][][1][0]{$\bs{z}_3$}
		\psfrag{TAGs4}[c][][1][0]{$\bs{z}_4$}
		\psfrag{TAGl0}[c][][1][0]{$\|\w\|_1=1$}
		\psfrag{TAGl1}[l][][1][0]{$\|\w\|_1=\lambda_1$}
		\psfrag{TAGl2}[l][][1][0]{$\|\w\|_1=\lambda_3$}
		\psfrag{TAGl3}[l][][1][0]{$\|\w\|_1=\lambda_2$}		
\includegraphics[width=0.95\linewidth]{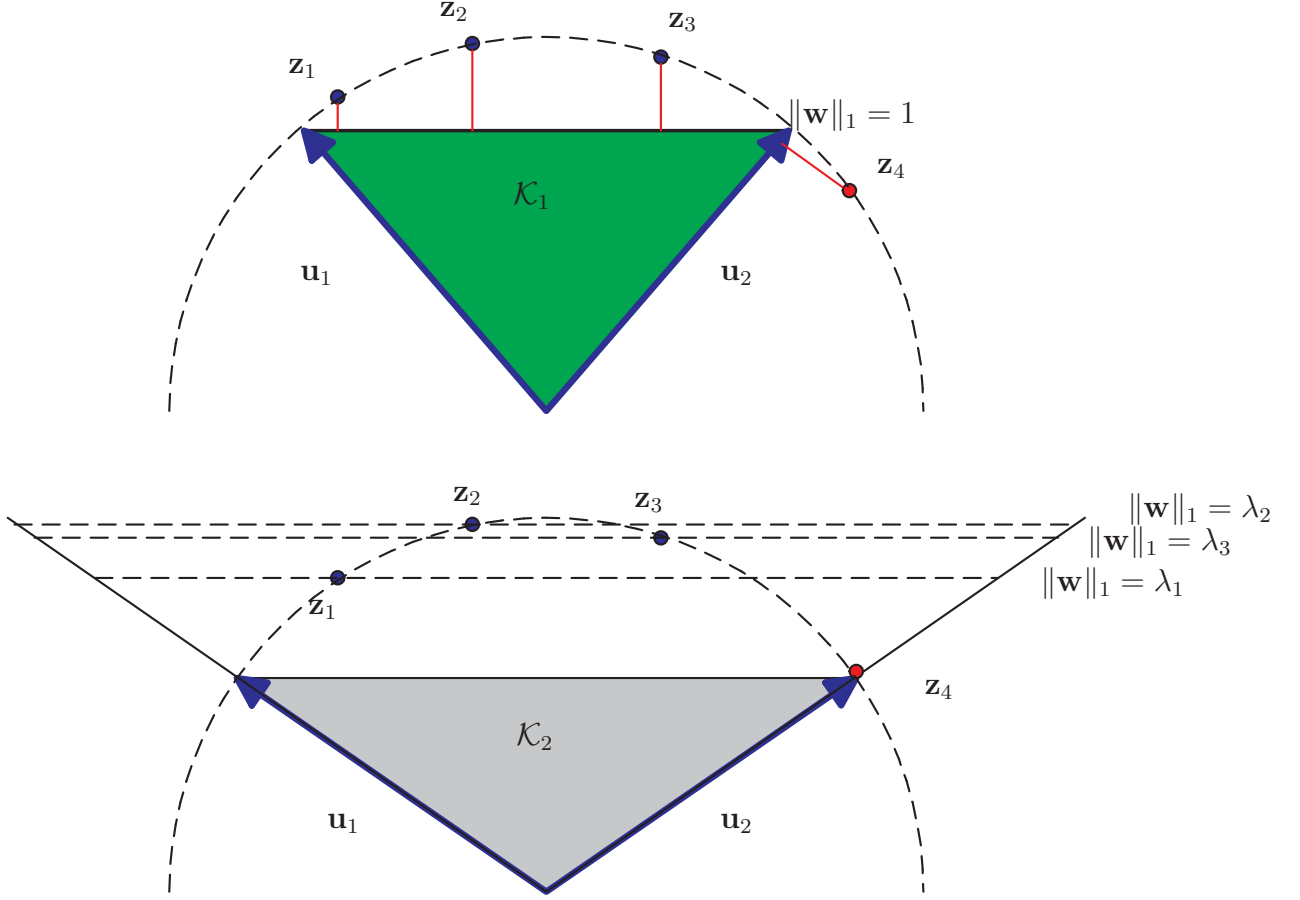}
\caption{\label{FIG_REG} Illustration of different cone solutions. For simplicity, we plot for the case $c=1$ and $\|\z\|_2\approx 1$.}
\end{figure}

In summary, the proposed formulation can be explicitly written as
	\begin{eqnarray}
		\label{EQU_MIN_RISK}
		\min_{\U} \left\{ \frac{1}{P} \sum_{i=1}^{P} \frac{1}{n_{q_i}}\left(\sum_{j=1}^{n_{q_i}} \min_{\w^{q_i}_j} \phi^{q_i}_j \| \z^{q_i}_j - \U\w^{q_i}_j \|_2^2 \right) \right\}\\
			\mbox{s.t.} \ \|\u_k\|_2 =c, \w^{q_i}_j\geq \bs{0}, 0 <\lambda_l \leq \|\w^{q_i}_j\|_1 \leq \lambda_u. \nonumber
	\end{eqnarray}

\subsection{Generalization bound}
We restrict our study on generalization bound from an algorithmic stability viewpoint, which is initially introduced in \cite{Bousquet_Elisseff02} and based on the concentration property of random variables. In the ranking context, generalization bounds for point-wise ranking / ordinal regression have been obtained \cite{Agarwal08,Cossock_Zhang08}. Recently, \cite{Lan_etal08} show that the generalization bound result in \cite{Bousquet_Elisseff02} still holds in the ranking context. More specifically, we would like to study the variation of the expected query-level risk, defined as 
	\begin{eqnarray}
	 	R(\mathcal{K}) = \int_{\mathcal{Q}\times \mathcal{Y}} L(\mathcal{K};q) {P}_{\mathcal{Q}}(dq).
	\end{eqnarray}
where $L(\mathcal{K};q)$ denotes the expected query-level loss defined as
	\begin{eqnarray}
	 	L(\mathcal{K};q) = \int_{\mathcal{Z}} l(\mathcal{K};\z^q,y^q) {P}_{\mathcal{Z}}(d\z^q)
	\end{eqnarray}
and ${P}_{\mathcal{Z}}$ denotes the probability distribution of the (ordered) document differences.

Following \cite{Bousquet_Elisseff02} and \cite{Lan_etal08} we define the uniform leave-one-query-out document-pair-level stability as
	\begin{eqnarray}
		\label{EQU_STABILITY}
	 	\beta = \sup_{q\in\mathcal{Q},i\in[1,\ldots,P]} |l(\mathcal{K}_{S};\z^q,y^q) - l(\mathcal{K}_{S^{-i}};\z^q,y^q)  |
	\end{eqnarray}
where $\mathcal{K}_{S}$ and $\mathcal{K}_{S^{-i}}$ are respectively the polyhedral cones learned from the full training set and that without the $i$th query. As stated in \cite{Lan_etal08}, it can be easily shown the following query-level stability bounds by integration or average sum of the term on the left hand side in the above definition
	\begin{eqnarray}
	 	|L(\mathcal{K}_{\mathcal{S}};q) - L(\mathcal{K}_{\mathcal{S}^{-i}};q) | \leq \beta, \forall i\\
		|\hat{L}(\mathcal{K}_{\mathcal{S}};q) - \hat{L}(\mathcal{K}_{\mathcal{S}^{-i}};q) | \leq \beta, \forall i.
	\end{eqnarray}
Using the above query-level stability results and by considering $S_{q_i}$ as query-level samples, one can directly apply the result in \cite{Bousquet_Elisseff02} (see also \cite{Lan_etal08}) to obtain the following generalization bound
\begin{theorem}
 For the proposed ConeRank algorithm with uniform leave-one-query-out document-pair-level stability $\beta$, with probability of at least $1-\varepsilon$ it holds
	\begin{eqnarray}
	 R(\mathcal{K}_{S}) \leq \hat{R}(\mathcal{K}_S) + 2\beta + (4P\beta + \gamma)\sqrt{\frac{\ln(1/\varepsilon)}{2P}},
	\end{eqnarray}
where $\gamma = \sup_{q\in\mathcal{Q}}l(\mathcal{K}_{S};\z^q,y^q)$ and $\varepsilon\in[0,1]$.
\end{theorem}

As can be seen, the bound on the expected query-level risk depends on the stability. It is of practical interest to study the stability $\beta$ for the proposed algorithm. The following result shows that the change in the cone due to leaving one query out can provide an effective upper bound on the uniform stability $\beta$.  For notational simplicity, we only consider the non-weighted version of the loss, as the weighted version is simply a scale of the bound by the maximum weight.

\begin{theorem}
Denote as $\U$ and $\U^{-i}$ the `basis' vectors of the polyhedral cones $\mathcal{K}_{S}$ and $\mathcal{K}_{S^{-i}}$ respectively. For a  ConeRank algorithm with non-weighted loss, we have
	\begin{eqnarray}
	\beta \leq 2 s_{\rm max} \lambda_u(\rho+\sqrt{K}c\lambda_u) + s_{\rm max}^2\lambda_u^2,
	\end{eqnarray}
	where $s_{\rm max} = \max_{i} \|\U -\U^{-i} \|$, $\| \bullet\|$ denotes the spectral norm, and $\rho$ is the normalizing factor of $\z$ (c.f. (\ref{EQU_NORM})).
\end{theorem}

\begin{proof}
Following the proposed algorithm, we equivalently study the bound of 
	\begin{eqnarray}
	 	\beta &= &   \sup_{q\in\mathcal{Q} \atop\|\z^q\|_2 \leq \rho}  \left| \min_{\w\in\mathcal{C}}\|\z^q-\U \w\|_2^2  
		 - \min_{\w\in\mathcal{C}} \|\z^q-\U^{-i}\w\|_2^2      \right| \nonumber
	\end{eqnarray}
where the constraint set $\mathcal{C}=\{ \w: \w \geq \bs{0}, \lambda_l \leq \| \w \|_1\leq \lambda_u \}$. Without loss of generality, we can assume that
\[\min_{\w\in\mathcal{C}}\|\z^q-\U \w\|_2^2 > \min_{\w\in\mathcal{C}} \|\z^q-\U^{-i}\w\|_2^2 \] 
and the minima are attained at $\w$ and $\w^{-i}$ respectively. Due to the definition, it follows that
	\begin{eqnarray}
	 	\beta & \leq & \sup_{q\in\mathcal{Q} \atop\|\z^q\|_2 \leq \rho} \left( 
			\|\z^q - \U\w^{-i}\|_2^2 - \|\z^q - \U^{-i}\w^{-i}\|_2^2 
		\right). 		\nonumber
	\end{eqnarray}
Expanding the term on the left, and using matrix norm inequalities, one obtains
	\begin{eqnarray}
	 	\beta & \leq & \sup_{q\in\mathcal{Q}} \left (2\|\U\|\|\bm{\Delta}\| + \|\bm{\Delta}\|^2) \|\w^{-i}\|_2^2 \right.\nonumber \\
			&& \left. + 2 \|\z^q\|_2 \|\bm{\Delta} \| \|\w^{-i}\|_2 \right) 
			%&\leq & 2\|\bm{\Delta} \|\lambda_u(1+\lambda_u) + \|\bm{\Delta} \|^2\lambda_u,
	\end{eqnarray}
where $\bm{\Delta} = \U-\U^{-i}$. The proof follows by the following facts
\begin{itemize}
	\item $\| \U\| \leq \sqrt{K}c$ due to each $\|\u_k\|_2\leq c$ and that $\| \U\|\leq \| \U\|_F$ where $\|\bullet\|_F$ denotes the Frobenius norm.
	\item 	$\|\w\|_2^2 \leq {\|\w\|_1^2}$ for $\w \geq \bs{0}$
	\item $\|\z^q\|_2\leq \rho$ due to the normalization
\end{itemize}
and that $\|\bm{\Delta}\|\leq s_{\rm max}$ by definition.
\end{proof}

It is more interesting to study the bound on $s_{\rm max}$. We conjecture that this will depend on the sample size as well as the nature of the proposed conic regularization. However, this is still an open question and such an analysis is beyond the scope of the current work.

We note importantly that as the stability bound can be made small by lowering $\lambda_u$. Doing so definitely improves stability at the cost of making the empirical risk large and hence the bias becomes significantly undesirable. In practice, it is important to select proper values of the parameters to provide optimal bias-variance trade-off. Next, we turn the discussion on practical implementation of the ideas, taking into account the large-scale nature of the problem.

\section{Implementation}

In the original formulation (\ref{EQU_MIN_RISK}), the scaling ambiguity is resolved by placing a norm constraint on $\u_k$. However, a direct implementation seems difficult. In what follows, we propose an alternative implementation by resolving the ambiguity on $\w$ instead. We fix $\|\w\|_1=1$ and consider the norm inequality constraint on $\u_k$ as $\|\u_k\|_2\leq c$ (i.e. convex relaxation on equality constraint) where $c$ is a constant of $\mathcal{O}(\|\z^q\|_2)$. This leads to an {\it approximate} formulation
	\begin{eqnarray}
		\label{EQU_MIN_RISK2}
		\min_{\U, \w^{q_i}_j}  \left\{ \frac{1}{P} \sum_{i=1}^{P} \frac{1}{n_{q_i}}\left(\sum_{j=1}^{n_{q_i}}\phi^{q_i}_j \| \z^{q_i}_j - \U\w^{q_i}_j \|_2^2 \right) \right\}\\
			\mbox{s.t.} \ \|\u_k\|_2 \leq c, \w^{q_i}_j\geq \bs{0},  \|\w^{q_i}_j\|_1 = 1. \nonumber
	\end{eqnarray}
The advantage of this approximation is that the optimization problem is now convex with respect to each $\u_k$ and still convex with respect to each ${\w^{q_i}_j}$. This suggests an alternating and iterative algorithm, where we only vary a subset of variables and fix the rest. The objective function should then always decrease. As the problem is not strictly convex, there is no guarantee of a global solution. Nevertheless, a locally optimal solution can be obtained. The additional advantage of the formulation is that gradient-based methods can be used for each sub-problem and this is very important in large-scale problems.

\begin{algorithm}[h]
   \caption{Stochastic Gradient Descent}
   \label{alg:sgd}
\begin{algorithmic}
   \STATE {\bfseries Input:} queries $q_i$ and pair differences $\z^{q_i}_j$.
   \STATE Randomly initialize $\u_{k}, \ \forall k\leq K$; set $\mu>0$
   \REPEAT
	\STATE 1. The \emph{folding-in} step (fixed $\U$):
	\STATE Randomly initialize $\w^{q_i}_j: \w^{q_i}_j\ge\bs{0}; \|\w^{q_i}_j\|_1=1$; 
	\REPEAT			
				\STATE 1a. Compute $\w^{q_i}_j\leftarrow \w^{q_i}_j-\mu{\partial \hat{R}(\w^{q_i}_j})/{\partial \w^{q_i}_j}$ 
				\STATE 1b. Set $\w^{q_i}_j \leftarrow \max\{\w^{q_i}_j,\bs{0}\}$ (element-wise)
				\STATE 1c. Normalize $\w^{q_i}_j \leftarrow \w^{q_i}_j / \|\w^{q_i}_j\|_1$			
	\UNTIL{converged}
	\STATE 2. The \emph{basis-update} step (fixed $\w$): 
	\FOR{$k=1$ {\bfseries to} $K$}
		\STATE 2a. Update $\u_k\leftarrow \u_k - \mu{\partial \hat{R}(\u_k})/{\partial \u_k}	 $
		\STATE 2b. Normalize $\u_k$ to norm $c$ if violated.
	\ENDFOR
	
   \UNTIL{converged}
\end{algorithmic}
\end{algorithm}

\subsection{Stochastic Gradient }

Since the number of pairs may be large for typically real datasets, we do not want to store every $\w^q_j$. Instead, for each iteration, we perform a \emph{folding-in} operation, in that we fix the basis $\U$, and estimate the coefficients $\w^q_j$. Since this is a convex problem, it is possible to apply the stochastic gradient (SG) method as shown in Algorithm \ref{alg:sgd}. Note that we express the empirical risk as the function of {\it only} variable of interest when other variables are fixed for notational simplicity. In practice, we also need to check if the cone is proper and we find this is always satisfied.

\subsection{Exponentiated Gradient}
Exponentiated Gradient (EG) \cite{kivinen1997exponentiated} is an algorithm for estimating distribution-like parameters. Thus, Step 1a 
can be replaced by
\begin{eqnarray}	
	\w^{q_i}_j\leftarrow \w^{q_i}_j\exp\left\{-\mu{\partial \hat{R}(\w^{q_i}_j)}/{\partial \w^{q_i}_j}\right\}
	(\mbox{element-wise}). \nonumber 
\end{eqnarray}
For faster numerical computation (by avoiding the exponential), as shown in \cite{kivinen1997exponentiated}, this
step can be approximated by 
	\begin{eqnarray}
		(\w^{q_i}_j)_k\leftarrow (\w^{q_i}_j)_k \left(1 - \mu \left({\partial \hat{R}(\hat{\z}^{q_i}_j)}/{\partial \hat{\z}^{q_i}_j}\right)^{\top} \ (\u_k-\hat{\z}^{q_i}_j)    \right) \nonumber 
	\end{eqnarray}
where the empirical risk $\hat{R}$ is parameterized in terms of $\hat{\z}^{q_i}_j=\U\w^{q_i}_j$. When the learning rate $\mu$ is sufficiently small, this update readily ensures the normalization of $\w^{q_i}_j$. The main difference between SG and EG is that, update in SG is \emph{additive}, while it is \emph{multiplicative} in EG.

\begin{algorithm}[h]
   \caption{Query-level Prediction}
   \label{alg:prediction}
\begin{algorithmic}
   \STATE {\bfseries Input:} New query $q$ with pair differences $\{\z^q_j\}_{j=1}^{n_q}$
   \STATE Maintain a scoring array $A$ of all pre-computed feature vectors,
initialize $A_{l}=0$ for all $l$. 
	\STATE Set $\phi^{q}_j=1,\forall j\leq n_q$.
   
   \FOR{$j=1$ {\bfseries to} $n_q$}
		\STATE Perform \emph{folding-in} to estimate the coefficients without the non-negativity
constraints.
		\STATE Check if the sum of the coefficients is positive, then $A_{l}\leftarrow A_{l}+1$
; otherwise $A_{m}\leftarrow A_{m}+1$ 
   \ENDFOR   
   \STATE Output the ranking based on the scoring array $A$.
\end{algorithmic}
\end{algorithm}

\subsection{Prediction \label{sub:Prediction}}
Assume that the basis $\U=(\u_{1},\u_{2},...,\u_{K})$ has been learned
during training. In testing, for each query, we are also given a set
of feature vectors, and we need to compute a ranking function that
outputs the appropriate positions of the vectors in the list. 

Unlike the training data where the order of the pair $(l,m)$ is given,
now this order information is missing. This breaks down the conic
assumption, in that the difference of the two vectors is the non-negative
combination of the basis vectors. Since the either preference orders
can potentially be incorrect, we relax the constraint of the non-negative
coefficients. The idea is that, if the order is correct, then the
coefficients are mostly positive. On the other hand, if the order
is incorrect, we should expect that the coefficients are mostly negative.
The query-level prediction is proposed as shown in Algorithm \ref{alg:prediction}. As this query-level prediction is performed over a query, it can address the shortcoming of logical discrepancy of document-level prediction in the pairwise approach.

\section{Discussion}

RankSVM \cite{joachims2002optimizing} defines the following loss
function over ordered pair differences
\begin{eqnarray*}
	L(\u)& = &\frac{1}{P}\sum_{j}\max(0,1-\u^{\top}\z_j) +\frac{C}{2}\|\u\|_2^2
\end{eqnarray*}
where $\u\in\mathbb{R}^{N}$ is the parameter vector, $C>0$ is the penalty constant and $P$ is the number of data pairs. 

Being a pairwise approach, RankNet instead uses 
	\begin{eqnarray*}
		L(\u) & = & \frac{1}{P}\sum_{j}\log(1+\exp\{-\u^{\top}\z_j\})+\frac{C}{2}\|\u\|_2^2.
	\end{eqnarray*}
%As $L(\u)>0$ in this case, there is no need for the truncation as in the SVM setting. 

This is essentially the 1-class SVM applied over the ordered pair differences. The quadratic regularization term tends to push the separating hyperplane away from the origin, i.e. maximizing the 1-class margin.

% Here we can see the relationship with the RankSVM and the RankNet in that we also encourage the \emph{positive correlation} between the scoring function $s_{l}=\u^{\top}\x_l$ and the relevance $r_{l}$ for a document $\x_l$. 

It can be seen that the RankSVM solution is the special case when the cone approaches a halfspace. In the original RankSVM algorithm, there is no intention to learn a non-negative subspace where ordinal information is to be found like in the case of ConeRank. This could potentially give ConeRank more analytical power to trace the origin of preferences.

\section{Experiments}

\subsection{Data and Settings}

We run the proposed algorithm on the latest and largest benchmark data LETOR 4.0. This has two data sets for supervised learning, namely MQ2007 (1700 queries) and MQ2008 (800 queries). Each returned document is assigned a integer-valued relevance score of $\{0,1,2\}$ where $0$ means that the document is irrelevant with respect to the query. For each query-document pair, a vector of $46$ features is pre-extracted, and available in the datasets. Example features include the term-frequency and the inverse document frequency
in the body text, the title or the anchor text, as well as link-specific like the PageRank and the number of in-links.  The data is split into a training set, a validation set and a test set.
We normalize these features so that they are roughly distributed as Gaussian with zero means and unit standard deviations. During the folding-in step, the parameters $\w^{q}_j$ corresponding to pair $j$th of query $q$ are randomly initialized from the non-negative uniform distribution and then normalized so that $\|\w^{q}_j\|_1 = 1$. The basis vectors $\u_k$ are randomly initialized to satisfy the relaxed norm constraint. The learning rate is $\mu=0.001$ for the SG and $\mu=0.005$ for the EG. For normalization, we select $\alpha=1$ and $\rho=\sqrt{N}$ where $N$ is the number of features, and we set $c=2\rho$.

\begin{figure}[t]
\begin{centering}
\includegraphics[width=0.90\linewidth]{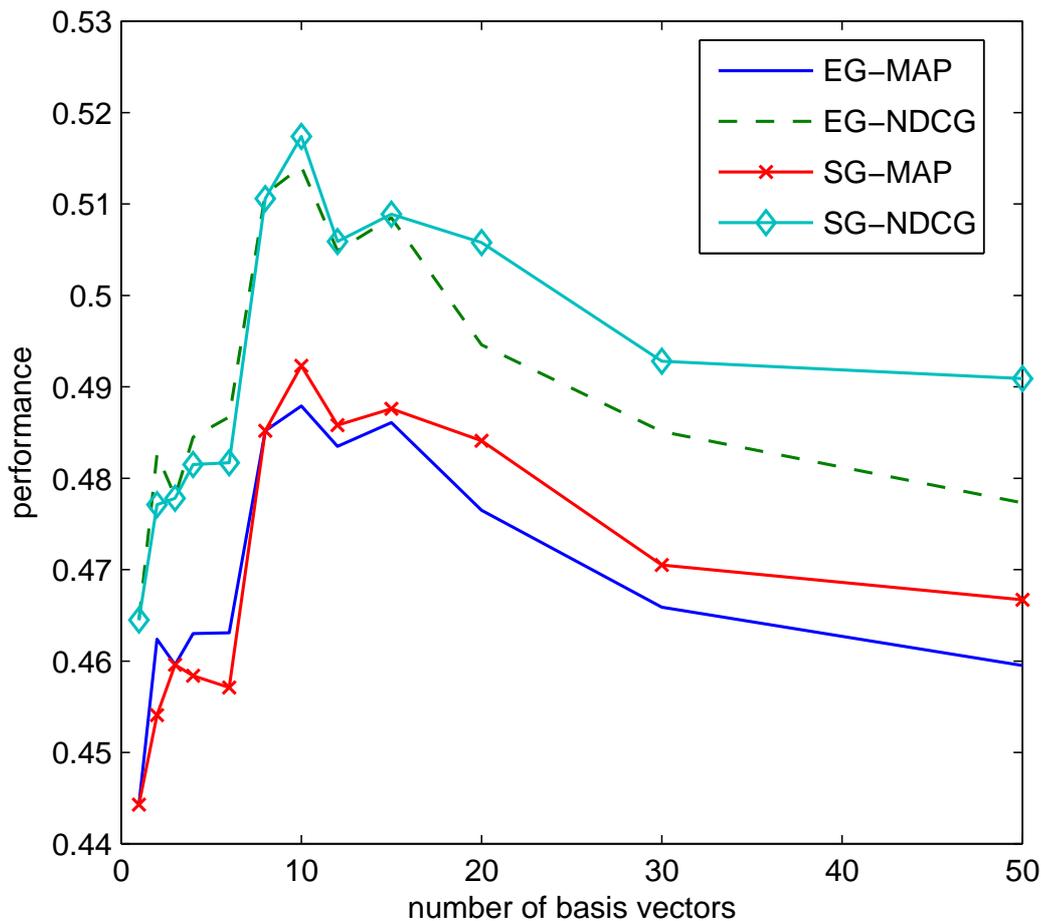}
\par\end{centering}

\vskip -0.1in
\caption{\label{Flo:basis-rank}Performance versus basis number}
\vskip -0.2in

\end{figure}

\subsection{Results}

The two widely-used evaluation metrics employed are the Mean Average Precision (MAP) and the Normalized Discounted Cumulative Gain (NDCG). We use the evaluation scripts distributed with LETOR 4.0.

In the first experiment, we investigate the performance of the proposed method with respect to the number of basis vectors $K$. The result of this experiment on the MQ2007 dataset is shown in Fig. \ref{Flo:basis-rank}. We note an interesting observation that the performance is highest at about $K=10$ out of 46 dimensions of the original feature space. This seems to suggest that the idea of capturing an informative subspace using the cone makes sense on this dataset. Furthermore, the study on the eigenvalue distribution of the non-centralized ordered pairwise differences on on the MQ2007 dataset, as shown in Fig. \ref{FIG_EIGEN}, also reveals that this is about the dimension that can capture most of the data energy.

\begin{figure}[h]
	\begin{centering}
	\includegraphics[width=0.9\linewidth]{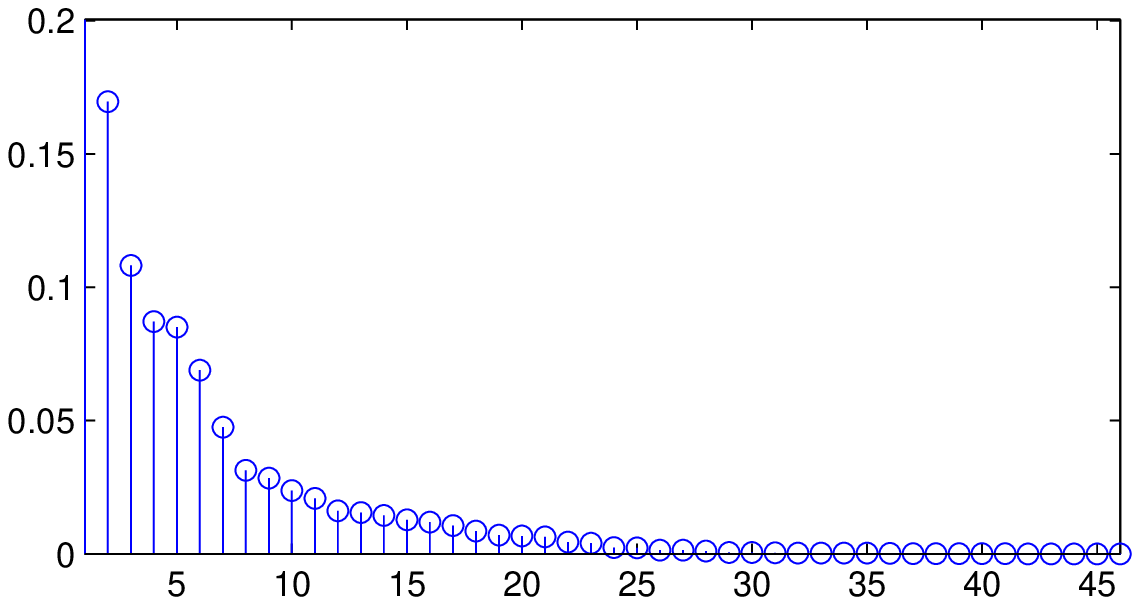}
	\end{centering}
	\caption{ \label{FIG_EIGEN} Eigenvalue distribution on the MQ2007 dataset.}	
\end{figure}

We then compare the proposed and recent base-line methods\footnote{from http://research.microsoft.com/en-us/um/beijing/projects/letor/letor4baseline.aspx} in the literature and the results on the MQ2007 and MQ2008 datasets are shown in Table \ref{Flo:baseline}. The proposed ConeRank is studied with $K=10$ due to the previous experiment. We note that all methods tend to perform better on MQ2007 than MQ2008, which can be explained by the fact that the MQ2007 dataset is much larger than the other, and hence provides better training. 

On the MQ2007 dataset, ConeRank compares favourably with other methods. For example, ConeRank-SG achieves the highest MAP score, whilst its NDCG score differs only less than 2\% when compared with the best (RankSVM-struct). On the MQ2008 dataset, ConeRank still maintains within the 3\% margin of the best methods on both MAP and NDCG metrics.

\begin{table}[h]
\caption{Results on LETOR 4.0.}
\label{Flo:baseline}
\vskip 0.15in
\begin{center}
\begin{small}
\begin{sc}

\begin{tabular}{l>{\centering}p{8mm}>{\centering}p{8mm}c>{\centering}p{8mm}>{\centering}p{8mm}}
\hline 
  & \multicolumn{2}{c}{MQ2007} & \multicolumn{2}{c}{MQ2008} \tabularnewline
\hline
Algorithms & MAP & NDCG & MAP & NDCG \tabularnewline
\hline

AdaRank-MAP & 0.482 & 0.518 & 0.463 & {\bf 0.480} \tabularnewline

AdaRank-NDCG & 0.486 & 0.517 & 0.464 & 0.477 \tabularnewline

ListNet & 0.488 & 0.524 & 0.450 & 0.469\tabularnewline

RankBoost & 0.489 & 0.527 & {\bf 0.467} & {\bf 0.480} \tabularnewline

RankSVM-struct & 0.489 & {\bf 0.528} & 0.450 & 0.458 \tabularnewline
\hline 
ConeRank-EG & 0.488 &  0.514  & 0.444 & 0.456\tabularnewline

ConeRank-SG & 0.{\bf 492} & 0.517  & 0.454 & 0.464\tabularnewline
\hline
\end{tabular}
\end{sc}
\end{small}
\end{center}
\vskip -0.1in

\end{table}

\section{Conclusion}
We have presented a new view on the learning to rank problem from a generalized inequalities perspective. We formulate the problem as learning a polyhedral cone that uncovers the non-negative subspace where ordinal information is found. A practical implementation of the method is suggested which is then observed to achieve comparable performance to state-of-the-art methods on the LETOR 4.0 benchmark data. 

There are some directions that require further research, including a more rigorous study on the bound of the spectral norm of the leave-one-query-out basis vector difference matrix, a better optimization scheme that solves the original formulation without relaxation, and a study on the informative dimensionality of the ranking problem.

% \section{Future Work}
% \begin{itemize}
% \item Directly optimising the evaluation metrics (MAP and NDCG).
% \item Better prediction: currently the prediction and the training do not
% match well.
% \item Clustering of preferences
% \item Mixture of preferences
% \item Prove the optimality of the proposed prediction method.
% \item Multiple cones: a single cone may not be rich enough to capture the
% ordering.
% \end{itemize}

\bibliographystyle{plain}

\end{document}